\documentclass[opre,nonblindrev]{informs4} 

\DoubleSpacedXI 

\usepackage{algorithm}
\usepackage{algorithmic}
\usepackage[caption=false]{subfig}
\usepackage{amsmath}
\usepackage{newfloat}
\usepackage{listings}
\floatstyle{ruled}
\newfloat{listing}{tb}{lst}{}
\floatname{listing}{Listing}

\usepackage{custom}
\usepackage{amsfonts}
\usepackage{comment}

\usepackage{natbib}
 \NatBibNumeric
 \def\bibfont{\small}%
 \bibpunct[, ]{[}{]}{,}{n}{}{,}%

\usepackage[colorlinks=true,breaklinks=true,bookmarks=true,urlcolor=blue,
     citecolor=blue,linkcolor=blue,bookmarksopen=false,draft=false]{hyperref}

\def\EMAIL#1{\href{mailto:#1}{#1}}

\TheoremsNumberedThrough     

\EquationsNumberedThrough    

\newcommand{\norm}[1]{\left\lVert#1\right\rVert}

\usepackage{algorithm}
\usepackage{algorithmic}
\usepackage{subfig}
\usepackage{amsmath}
\usepackage{newfloat}
\usepackage{listings}
\floatstyle{ruled}
\newfloat{listing}{tb}{lst}{}
\floatname{listing}{Listing}

\usepackage{amsfonts}
\usepackage{comment}

\usepackage[colorlinks=true,breaklinks=true,bookmarks=true,urlcolor=blue,
     citecolor=blue,linkcolor=blue,bookmarksopen=false,draft=false]{hyperref}

\usepackage{endnotes}
\let\footnote=\endnote
\let\enotesize=\normalsize
\def\notesname{Endnotes}%
\def\makeenmark{$^{\theenmark}$}
\def\enoteformat{\rightskip0pt\leftskip0pt\parindent=1.75em
  \leavevmode\llap{\theenmark.\enskip}}

\usepackage{natbib}
 \bibpunct[, ]{(}{)}{,}{a}{}{,}%
 \def\bibfont{\small}%
\begin{document}




\TITLE{The Role of Lookahead and Approximate Policy Evaluation in Reinforcement Learning with Linear Value Function Approximation}

\ARTICLEAUTHORS{%
\AUTHOR{Anna Winnicki}
\AFF{Department of Electrical and Computer Engineering and Coordinated Science Laboratory, University of Illinois at Urbana-Champaign, Urbana, IL 61801, \EMAIL{annaw5@illinois.edu}} 
\AUTHOR{Joseph Lubars}
\AFF{Sandia National Laboratories, 1515 Eubank Blvd SE, Albuquerque, NM 87123, \EMAIL{lubars2@illinois.edu}} 
\AUTHOR{Michael Livesay}
\AFF{Sandia National Laboratories, 1515 Eubank Blvd SE, Albuquerque, NM 87123, \EMAIL{mlivesa@sandia.gov}} 
\AUTHOR{R. Srikant}
\AFF{Department of Electrical and Computer Engineering and Coordinated Science Laboratory, University of Illinois at Urbana-Champaign, Urbana, IL 61801, \EMAIL{rsrikant@illinois.edu}. R. Srikant is also affiliated with c3.ai DTI.}
} 

\ABSTRACT{%
Function approximation is widely used in reinforcement learning to handle the computational difficulties associated with very large state spaces. However, function approximation introduces errors which may lead to instabilities when using approximate dynamic programming techniques to obtain the optimal policy. Therefore, techniques such as lookahead for policy improvement and $m$-step rollout for policy evaluation are used in practice to improve the performance of approximate dynamic programming with function approximation. We quantitatively characterize, for the first time, the impact of lookahead and $m$-step rollout on the performance of approximate dynamic programming (DP) with function approximation: (i) without a sufficient combination of lookahead and $m$-step rollout, approximate DP may not converge, (ii) both lookahead and $m$-step rollout improve the convergence rate of approximate DP, and (iii) lookahead helps mitigate the effect of function approximation and the discount factor on the asymptotic performance of the algorithm. Our results are presented for two approximate DP methods: one which uses least-squares regression to perform function approximation and another which performs several steps of gradient descent of the least-squares objective in each iteration.
}%


\KEYWORDS{Markov Decision Processes, Dynamic Programming} 
%

  \newpage
  \theARTICLETITLE
  \theARTICLEAUTHORS
  \theARTICLEABSTRACT
  \theARTICLERULE


%


\section{Introduction} 
In many applications of reinforcement learning, such as playing chess and Go, the underlying model is known and so the main challenge is in solving the associated dynamic programming problem in an efficient manner. Policy iteration and variants of policy iteration \cite{bertsekas2019reinforcement,Bertsekas2011ApproximatePI,bertsekastsitsiklis} that solve dynamic programming problems rely on computations that are infeasible due to the sizes of the state and action spaces in modern reinforcement learning problems.
As a remedy to this ``curse of dimensionality,'' several state-of-the-art algorithms \cite{silver2017shoji, silver2017mastering, DBLP:journals/corr/MnihBMGLHSK16} employ function approximation, lookahead for policy improvement, $m$-step rollout for policy evaluation, and gradient descent to compute the function approximation, see Section \ref{section2} for a definition of these terms. 

Our goal in this paper is to understand the role of multi-step lookahead for policy improvement (i.e., repeatedly applying the Bellman operator multiple times) and $m$-step rollout (which is a technique to approximately evaluate a policy by rolling out the dynamic programming tree for a certain number of steps $m$, see Section \ref{section2} for definitions of these terms) on the accuracy of approximate policy iteration techniques with linear value function approximation. The algorithms we study in this paper are closely related to least-squares policy iteration (LSPI) \cite{parr, bucsoniu2012least, lagoudakis2001model} and approximate policy iteration (PI), see  \cite{bertsekastsitsiklis, bertsekas2019reinforcement}. In the analysis of approximate PI, it is assumed that the policy evaluation and improvement steps have bounded errors, and using these, an error bound is obtained for the algorithm which repeatedly uses approximate policy evaluation and improvement. LSPI is an algorithm that builds on approximate PI where the policy evaluation step uses a least-squares algorithm to estimate the value function for the entire state space using the value function evaluated at a few states. However, the bounds presented in \cite{parr} as well as the related studies in \cite{lagoudakis2001model, bucsoniu2012least}, are simply a special case of the bounds for generic approximate PI \cite{bertsekastsitsiklis, bertsekas2019reinforcement}, and do not explicitly take into account the details of the implementation of least-squares-based policy evaluation. When such details are taken into account, it turns out the roles of the depth of lookahead ($H$) and rollout ($m$) become important, and their impact on the error bounds on the performance of approximate value iteration has not been characterized in prior work. 

The recent work in \cite{efroni2019combine} considers a variant of policy iteration that utilizes lookahead and approximate policy evaluation using an $m$-step rollout. As stated in the motivation in \cite{efroni2019combine}, it is well known that Monte Carlo Tree Search (MCTS) \cite{kocisszepesvari, browne} works well in practice
even though the worst-case complexity can be exponential \cite{shah2020nonasymptotic}; see \cite{munosbook} for some analysis of MCTS in MDPs where the number of states that can be visited from a given state is bounded. Motivated by policy iteration, the algorithm in \cite{efroni2019combine} estimates the value function associated with a policy and aims to improve the policy at each step. Policy improvement is achieved by obtaining the ``greedy'' policy in the case of policy iteration or a lookahead policy in the work of \cite{efroni2019combine}, which involves applying the Bellman operator several times to the current iterate before obtaining the greedy policy. The idea is that the application of the Bellman operator several times gives a more accurate estimate of the optimal value function. Then, similarly to policy iteration, the algorithm in \cite{efroni2019combine} aims to evaluate the new policy. The algorithm in \cite{efroni2019combine} uses an $m$-step rollout to compute the value function associated with a policy, i.e., it applies the Bellman operator associated with the policy $m$ times. 
The work of \cite{efroni2019combine} establishes that a lookahead can significantly improve the rate of convergence if one uses the value function computed using lookahead in the approximate policy evaluation step. However, like the works of \cite{bertsekastsitsiklis, bertsekas2019reinforcement, parr, lagoudakis2001model, bucsoniu2012least}, the work of \cite{efroni2019combine} does not study the use of function approximation which is critical to handling large state spaces, nor does it quantify the effects of varying $m$ in the convergence of their algorithm. Our results show that the aforementioned results change drastically when least-squares-based policy evaluation is incorporated. 

In this paper, we assume that policies are evaluated at a few states using an $m$-step rollout. The use of a partial rollout in our algorithm is similar to modified policy iteration \cite{Puterman1978ModifiedPI}, which is also called optimistic policy iteration \cite{bertsekastsitsiklis}. However, motivated by \cite{TsitsiklisRoy}, we present an example which shows that the algorithm can diverge when function approximation is used. Therefore,
our goal is to understand how to integrate linear value function approximation into the well-studied modified policy iteration algorithm. To the best of our knowledge, none of the prior works consider the impact of using gradient descent to implement an approximate version of least-squares policy evaluation within approximate PI. Thus, our algorithm and analysis can be viewed as a detailed look at approximate PI and modified PI when linear function approximation, least-squares policy evaluation and gradient descent are used to evaluate policies.

Our contributions are as follows:

\begin{itemize}
    \item We examine the impact of lookahead and $m$-step rollout on approximate policy iteration with linear function approximation. As is common in practice, we assume that we evaluate an approximate value function only for some states at each iteration. We obtain performance bounds for our algorithm under the assumption that the sum of the lookahead and the number of steps in the $m$-step rollout is sufficiently large. We demonstrate through an extension of a counterexample in \cite{Tsitsiklis94feature-basedmethods} that such a condition is necessary, in general, for convergence with function approximation unlike the tabular setting in the prior works.  See Section \ref{subsection:counterexAppendix} for our counterexample.
    \item  For ease of exposition, we first present the case where one solves a least-squares problem at each iteration to obtain the weights associated with the feature vectors in the function approximation of the value function in Section \ref{subsection3.4}. Our performance bounds in this case generalize the bounds in \cite{parr, bertsekastsitsiklis, lagoudakis2001model, bucsoniu2012least}, \cite{bertsekas2019reinforcement}, and \cite{efroni2019combine} for approximate PI. 
    \item We then consider a more practical and widely-used scheme where several steps of gradient descent are used to update the weights of the value function approximation at each iteration. Obtaining performance bounds for the gradient descent algorithm is more challenging and these bounds can be found in Section \ref{SectionGD}.
    \item Our results show that the sufficient conditions on the hyperparameters (such as the amount of lookahead, rollout, gradient descent parameters) of the algorithm required for convergence either do not depend on the size of the state space or depend only logarithmically on the size of the state space. Our results also illustrate the role of feature vectors in the amount of lookahead required. \color{black}
    \item In addition to asymptotic performance bounds, we also provide finite-time guarantees for our algorithms. Our finite-time bounds show that our algorithm converges exponentially fast in the case of least-squares as well as the case where a fixed number of gradient descent steps are performed in each iteration of the algorithm.
\end{itemize}

\subsection{Other Related Work}

 The role of lookahead and rollout in improving the performance of RL algorithms has also been studied in a large number of papers including \cite{shahxie, moerland2020framework, efroni2020online, tomar2020multistep, efroni2018multiplestep, springenberg2020local, 9407870}. The works of \cite{baxter, veness, lanctot2014monte} explore the role of tree search in RL algorithms. However, to the best of our knowledge, the amount of lookahead and rollout needed as a function of the feature vectors has not been quantified in prior works. 

The works of \cite{Bertsekas2011ApproximatePI} and \cite{bertsekas2019reinforcement} also study a variant of policy iteration wherein a greedy policy is evaluated approximately using feature vectors at each iteration. These papers also provide rates of convergence as well as a bound on the approximation error. However, our main goal is to understand the relations between function approximation and lookahead/rollout which are not considered in these other works.

\section{Preliminaries} \label{section2}
We consider a Markov Decision Process (MDP), which is defined to be a 5-tuple $(\scriptS, \scriptA, P, r, \alpha)$. The finite set of states of the MDP is $\scriptS$. There exists a finite set of actions associated with the MDP $\scriptA$. Let $P_{ij}(a)$ be the probability of transitioning from state $i$ to state $j$ when taking action $a \in \scriptA$. We denote by $s_k$ the state of the MDP and by $a_k$ the corresponding action at time $k$. We associate with state $s_k$ and action $a_k$ a non-deterministic reward $r(s_k, a_k) \in [0, 1] \forall s_k \in \scriptS, a_k \in \scriptA.$ 

Our objective is to maximize the cumulative discounted reward with discount factor $\alpha \in (0, 1).$ 
Towards this end, we seek to find a deterministic policy $\mu$ which associates with each state $s\in \scriptS$ an action $\mu(s) \in \scriptA$. For every policy $\mu$ and every state $s \in \scriptS$ we define $J^{\mu}(s)$ as follows:
\begin{align*}
    J^{\mu}(s) := E[\sum_{i=0}^\infty \alpha^k r(s_k, \mu(s_k))|s_0=s].
\end{align*}  

We define the optimal reward-to-go $J^*$ as 
$J^*(s) := \underset{\mu}\max J^\mu(s).$ The objective is to find a policy $\mu$ that maximizes $J^\mu(s)$ for all $s \in \scriptS$. Towards the objective, we associate with each policy $\mu$ a function $T_\mu: \mathbb{R}^{|\scriptS|} \to \mathbb{R}^{|\scriptS|}$ where for $J \in \mathbb{R}^{|\scriptS|},$ the $s$th component of $T_{\mu}J$ is
\begin{align*}
(T_\mu J)(s) = r(s, \mu(s)) + \alpha \sum_{j=1}^{|\scriptS|} P_{sj}(\mu(s)) J(j), 
\end{align*} for all $s \in S$. If function $T_{\mu}$ is applied $m$ times to vector $J \in \mathbb{R}^{|\scriptS|},$ then we say that we have performed an $m$-step rollout of the policy $\mu$ and the result $T^m_\mu J$ of the rollout is called the return.

Similarly, we define the Bellman operator $T: \mathbb{R}^{|\scriptS|} \to \mathbb{R}^{|\scriptS|}$ with the $s$th component of $TJ$ being  
\begin{align}
(TJ)(s) = \underset{a \in \scriptA}\max \Bigg \{ r(s, a) + \alpha \sum_{j=1}^{|\scriptS|} P_{sj}(a)J(j) \Bigg \}. \label{T}
\end{align}
The policy corresponding to the $T$ operator is defined as the \textit{greedy} policy. If operator $T$ is applied $H$ times to vector $J \in \mathbb{R}^{|\scriptS|},$ we call the result - $T^H J$ - the $H$-step ``lookahead'' corresponding to $J$. The greedy policy corresponding to $T^H J$ is called the $H$-step lookahead policy, or the lookahead policy, when $H$ is understood. More precisely, given an estimate $J$ of the value function, the lookahead policy is the policy $\mu$ such that $T_\mu(T^{H-1} J)=T(T^{H-1} J).$

It is well known that each time the Bellman operator is applied to a vector $J$ to obtain $TJ,$ the following holds:
\begin{align*}
    \norm{TJ-J^*}_\infty\leq \alpha\norm{J-J^*}_\infty.
\end{align*} Thus, applying $T$ to obtain $TJ$ gives a better estimate of the value function than $J.$

The Bellman equations state that the vector $J^\mu$ is the unique solution to the linear equation 
\begin{align}
J^\mu = T_\mu J^\mu. \label{bellman}
\end{align}

Additionally, we have that $J^*$ is a solution to
\begin{align*}
J^* = TJ^*.
\end{align*}
Note that every greedy policy w.r.t. $J^*$ is optimal and vice versa \cite{bertsekastsitsiklis}. 

We will now state several useful properties of the operators $T$ and $T_\mu$. See \cite{bertsekastsitsiklis} for more on these properties. \color{black} Consider the vector $e \in \mathbb{R}^{|\scriptS|}$ where $e(i) = 1 \forall i \in 1, 2, \ldots, |\scriptS|.$ We have:
\begin{equation}
    T(J + ce) = TJ + \alpha ce, \quad T_\mu(J + ce) = T_\mu J + \alpha ce. \label{eq:usefulproperties}
\end{equation}
Operators $T$ and $T_\mu$ are also monotone:
\begin{align}
    J \leq J' \implies TJ \leq TJ', \quad T_\mu J \leq T_\mu J'. \label{monotonicityproperty}
\end{align}

\section{Approximate Policy Iteration With Linear Value Function Approximation}
\label{sectionMPI}
\begin{algorithm}[tb]
\caption{Approximate Policy Iteration With Lookahead}
\label{alg:genalg}
\textbf{Input}: $\theta_0, m,$ $H.$
\begin{algorithmic}[1] 
\STATE Let $k=0$.
\STATE Let $\mu_{k+1}$ be such that $\norm{T^H J_k - T_{\mu_{k+1}} T^{H-1} J_k}_\infty \leq \eps_{LA}$.\\
\STATE Compute $\theta_{k+1}$ such that $\hat{J}^{\mu_{k+1}} := \Phi \theta_{k+1}$ satisfies the following:
\begin{align*}
\norm{\hat{J}^{\mu_{k+1}} - J^{\mu_{k+1}}}_\infty \leq \delta,
\end{align*}
\STATE $J_{k+1} = \hat{J}^{\mu_{k+1}}.$
\STATE Set $k \leftarrow k+1.$ Go to 2.
\end{algorithmic}
\end{algorithm}

As mentioned in the Introduction, the work of \cite{efroni2019combine} extends the result of \cite{bertsekas2019reinforcement} to incorporate the use of lookahead policies, as opposed to 1-step greedy policies as well as $m$-step returns. We outline the Algorithm of \cite{efroni2019combine} in Algorithm \ref{alg:genalg}.  We then wish to incorporate linear value function approximation into the analysis. We will outline the approximate policy iteration algorithm with lookahead and linear value function approximation and compare it to Algorithm \ref{alg:genalg}.

\subsection{Approximate Policy Iteration With Linear Value Function Approximation} \label{sectionLS}
 
\begin{algorithm}[tb]
\caption{Least-Squares Function Approximation Algorithm}
\label{alg:LSalg}
\textbf{Input}: $J_0, m,$ $H,$ feature vectors $\{ \phi(i) \}_{i \in \scriptS}, \phi(i) \in \mathbb{R}^d$  and subsets $\scriptD_k \subseteq \scriptS, k = 0, 1, \ldots.$ Here $\scriptD_k$ is the set of states at which we evaluate the current policy at iteration $k.$
\begin{algorithmic}[1] 
\STATE Let $k=0$.
\STATE Let $\mu_{k+1}$ be such that $\norm{T^H J_k - T_{\mu_{k+1}} T^{H-1} J_k}_\infty \leq \eps_{LA}$.\\\label{step 2 alg}
\STATE Compute $\hat{J}^{\mu_{k+1}}(i) = T_{\mu_{k+1}}^m T^{H-1} (J_k)(i)+w_{k+1}(i)$ for $i \in \scriptD_k.$ \\ \label{step 3 alg}
\STATE Choose $\theta_{k+1}$ to solve 
\begin{align}
    \min_\theta \sum_{i \in D_k} \Big( (\Phi \theta)(i) - \hat{J}^{\mu_{k+1}}(i) \Big)^2, \label{step 4 alg}
\end{align} 
where $\Phi$ is a matrix whose rows are the feature vectors.
\STATE $J_{k+1} = \Phi \theta_{k+1}.$
\STATE Set $k \leftarrow k+1.$ Go to 2.
\end{algorithmic}
\end{algorithm}

Our main algorithm is described in Algorithm \ref{alg:LSalg}. We now explain our algorithm and the associated notation in detail. Due to the use of function approximation, our algorithm is an approximation to policy iteration with lookahead. At each iteration index, say, $k$, we have an estimate of the value function, which we denote by $J_k$. To obtain $J_{k+1}$, we perform a lookahead to improve the value function estimate at a certain number of states (denoted by $\scriptD_k$) which can vary with each iteration. For example, $\scriptD_k$ could be chosen as the states visited when performing a tree search to approximate the lookahead process. During the lookahead process, we note that we will also obtain an $H$-step lookahead policy, which we denote by $\mu_{k+1}$.   As noted in the Introduction, the computation of $T^{H-1}(J_k)(i)$ for $i \in \scriptD_k$ in Step \ref{step 3 alg} of Algorithm \ref{alg:LSalg} may be computationally infeasible; however, as noted in \cite{efroni2019combine}, techniques such as Monte Carlo tree search (MCTS)   are employed in practice to approximately estimate $T^{H-1}(J_k)(i).$ \color{black} In this paper, we model the fact that lookahead cannot be performed exactly due to the associated computational complexity by allowing an error in the lookahead process which we denote by $\eps_{LA}$ in Step~\ref{step 2 alg} of the algorithm. The use of $\eps_{LA}$ is similar to the work of \cite{efroni2019combine}.

We obtain estimates of $J^{\mu_{k+1}}(i)$ for $i \in \scriptD_k$, which we call $\hat{J}^{\mu_{k+1}}(i)$. To obtain an estimate of $J^{\mu_{k+1}}(i)$, we perform an $m$-step rollout with policy $\mu_{k+1}$, and obtain a noisy version of $T^m_{\mu_{k+1}}T^{H-1}J_k(i)$ for $i \in \scriptD_k.$ We also model the approximation error in the rollout by adding noise (denoted by $w_{k+1}(i)$ in Step ~\ref{step 3 alg} of the algorithm) to the return (result of the rollout - see Section \ref{section2}) computed at the end of this step. In order to estimate the value function for states not in $\scriptD_k$, we associate with each state $i \in \scriptS$  a feature vector $\phi(i)\in \mathbb{R}^d$ where typically $d << |\scriptS|$. The matrix comprised of the feature vectors as rows is denoted by $\Phi$. We use those estimates to find the best fitting $\theta \in \mathbb{R}^d$, i.e., 
\begin{align*}
    \min_\theta \sum_{i \in D_k} \Big( (\Phi \theta)(i) - \hat{J}^{\mu_{k+1}}(i) \Big)^2.
\end{align*} 
The solution to the above minimization problem is denoted by $\theta_{k+1}$. The algorithm then uses $\theta_{k+1}$ to obtain $J_{k+1} = \Phi \theta_{k+1}$. The process then repeats. This step of our algorithm differs from the algorithm in \cite{efroni2019combine} in that the algorithm in \cite{efroni2019combine} does not assume any particular technique for computing the estimate of $J^{\mu_{k+1}}$. It merely assumes the existence of some $\delta$ such that the distance from the estimate of $J^{\mu_{k+1}}$ to $J^{\mu_{k+1}}$ is less than $\delta$. We will show that the results of \cite{efroni2019combine} change drastically when linear function approximation is employed to estimate $J^{\mu_{k+1}}$. Additionally, note that to compute $\hat{J}^{\mu_{k+1}}(i),$ we obtain noisy estimates of $T_{\mu_{k+1}}^m T^{H-1} J_k (i)$ for $i\in \scriptD_k.$ Another alternative is to instead obtain noisy estimates of $T_{\mu_{k+1}}^m J_k (i)$ for $i\in \scriptD_k.$ It was shown in \cite{efroni2019combine} that the former option is preferable because it has a certain contraction property. Thus, we have chosen to use this computation in our algorithm as well.  However, we will show in Appendix \ref{appendix:thm2} that the algorithm also has bounded error which becomes small if $m$ is chosen to be sufficiently large.





\begin{remark}
We note that $\mu_{k+1}(i)$ in Step \ref{step 2 alg} of Algorithm \ref{alg:LSalg} does not have to be computed for all states $i\in \scriptS.$ The actions $\mu_{k+1}(i)$ have to be computed only for those $i\in\scriptS$ that are encountered in the rollout step of the algorithm (Step \ref{step 3 alg}).
\end{remark}

To analyze Algorithm \ref{alg:LSalg}, we make the following assumption which states that we explore a sufficient number of states during the policy evaluation phase at each iteration and that the noise is bounded.

\begin{assumption}\label{assume 1 or} 
For each $k \geq 0, \text{ rank }\{ \phi(i)\}_{i \in D_k} = d$.
Additionally, assume that the noise $w_k$ is bounded.
For some $\eps_{PE} >0,$ the noise in policy evaluation satisfies $\norm{w_k}_\infty \leq \eps_{PE} \forall k$. 
\end{assumption} 

Using Assumption~\ref{assume 1 or}, $J_{k+1}$ can be written as
\begin{align}
J_{k+1} &= \Phi \theta_{k+1} =\underbrace{\Phi(\Phi_{\scriptD_{k}}^\top \Phi_{\scriptD_{k}} )^{-1} \Phi_{\scriptD_{k}}^\top \scriptP_k}_{=: \scriptM_{k+1}} \hat{J}^{\mu_{k+1}},\label{defMk}\end{align} 
where $\Phi_{\scriptD_{k}}$ is a matrix whose rows are the feature vectors of the states in $\scriptD_{k}$ and $\scriptP_k$ is a matrix of zeros and ones such that $\scriptP_k\hat{J}^{\mu_{k+1}}$ is a vector whose elements are a subset of the elements of $\hat{J}^{\mu_{k+1}}$ corresponding to $\scriptD_k$. Note that $\hat{J}^{\mu_{k+1}}(i)$ for $i\notin\scriptD_k$ does not affect the algorithm, so we can define $\hat{J}^{\mu_{k+1}}(i)=T_{\mu_{k+1}}^m T^{H-1} J_k(i)$ for $i\notin\scriptD_k.$

Written concisely, our algorithm is as follows:
\begin{equation}
    J_{k+1} = \scriptM_{k+1} (T_{\mu_{k+1}}^m T^{H-1} J_k+w_k),
\end{equation}
where $\mu_{k+1}$ is defined in Step 2 of the algorithm. 
Since $w_k(i)$ for $i\notin\scriptD_k$ does not affect the algorithm, we define $w_k(i)=0$ for $i\notin\scriptD_k.$

We now present a counter-example to show that applying linear value function to approximate policy iteration is not a straightforward application of the bounds in \cite{efroni2019combine} and \cite{bertsekas2019reinforcement}. In the counter-example, we give an MDP which uses an $m$-step return to evaluate greedy policies at several states of the state space and linear value function approximation to estimate the value functions corresponding to the greedy policy at the rest of the states. The iterates diverge, which shows that more work needs to be done to understand how to incorporate linear value function approximation into approximate policy iteration.

\subsection{Counterexample} \label{subsection:counterexAppendix}

Even though, in practice, $J^{\mu_k}$ is what we are interested in, the values $J_k$ computed as part of our algorithm should not go to $\infty$ as the algorithm uses the values of $J_k$ to compute $J^{\mu_k}$  so divergence of $J_k$ can result in inaccurate computations of values of $J^{\mu_k}$. Additionally, divergence of $J_k$ would result in a numerically unstable algorithm, which is also undesirable. Here, we show that $J_k$ can become unbounded. 

The example we use is depicted in Figure \ref{fig:TsitVanRoyIm}. 
\begin{figure}
    \centering
    \subfloat[\centering $\mu^a$]{\includegraphics[width=2.5cm]{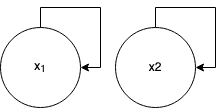} }%
    \qquad
    \subfloat[\centering $\mu^b$]{\includegraphics[width=2.5cm]{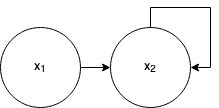} }%
     \caption{An example illustrating the necessity of the condition in Theorem~\ref{mainAPI}}%
    \label{fig:TsitVanRoyIm}%
\end{figure}
There are two policies, $\mu^a$ and $\mu^b$ and the transitions are deterministic under the two policies. The rewards are deterministic and only depend on the states. The rewards associated with states are denoted by $r(x_1)$ and $r(x_2),$
with $r(x_1)>r(x_2)$. Thus, the optimal policy is $\mu^a$. We assume scalar features $\phi(x_1)=1$ and $\phi(x_2)=2.$

We fix $H=1$.
The MDP follows policy $\mu^a$ when:
\begin{align*}
    &J_k(x_1) > J_k(x_2) \implies  \theta_k > 2\theta_k.
\end{align*} Thus, as long as $\theta_k>0,$ the lookahead policy will be $\mu_b.$ 

We will now show that $\theta_k$ increases at each iteration when $\delta_{FV} \alpha^{m+H-1}>1.$ We assume that $\theta_0>0$ and $\scriptD_k = \{1, 2\}$ $\forall k.$ A straightforward computation shows that $\delta_{FV}=\frac{6}{5}.$
At iteration $k+1,$ suppose $\mu_{k+1}=\mu^b,$ our $\hat{J}^{\mu_{k+1}}(i)$ for $i = 1, 2$ are as follows:
\begin{align*}
    \hat{J}^{\mu_{k+1}}(1) =r(x_1)+\sum_{i=1}^{m-1} r(x_1) \alpha^i + 2 \alpha^m \theta_k, \quad
    \hat{J}^{\mu_{k+1}}(2) = r(x_2) +\sum_{i=1}^{m-1} r(x_2)\alpha^i + 2 \alpha^m \theta_k.
\end{align*}
Thus, from Step \ref{step 4 alg} of Algorithm \ref{alg:LSalg}:
\begin{align*}
    &\theta_{k+1} = \arg \min_\theta \sum_{i =1}^2 \Big( (\Phi \theta)(i) -  \hat{J}^{\mu_{k+1}}(i) \Big)^2 \\
    &\implies \theta_{k+1} = \frac{\sum_{i=1}^{m-1} \alpha^i r(x_1)}{5} + \frac{2 \sum_{i=1}^{m-1} \alpha^i r(x_2)}{5} + \frac{6 \alpha^m \theta_k}{5}\\
    &\implies \theta_{k+1}> \frac{6}{5} \alpha^{m}\theta_k.
\end{align*} 
Thus, since $\theta_0 > 0$ and $H=1$, when $ \frac{6}{5} \alpha^{m+H-1}\theta_k =\delta_{FV} \alpha^{m+H-1} >1,$ $\theta_{k}$ goes to $\infty.$

It is worth noting, even though $J^{\mu_k}$ is always bounded, the fact that $J_k$ diverges means that the algorithm cannot be implemented in a numerically stable manner.

\subsection{Approximate Policy Iteration With Time-Dependent Policy Evaluation Error}

\begin{algorithm}[tb]
\caption{Modified Policy Iteration With Lookahead And Function Approximation}
\label{alg:genalg2}
\textbf{Input}: $\theta_0, m,$ $H.$
\begin{algorithmic}[1] 
\STATE Let $k=0$.
\STATE Let $\mu_{k+1}$ be such that $\norm{T^H J_k - T_{\mu_{k+1}} T^{H-1} J_k}_\infty \leq \eps_{LA}$.\\
\STATE Compute $\theta_{k+1}$ such that $\hat{J}^{\mu_{k+1}} := \Phi \theta_{k+1}$ satisfies the following:
\begin{align*}
\norm{\hat{J}^{\mu_{k+1}} - J^{\mu_{k+1}}}_\infty \leq \delta_k,
\end{align*}
\STATE $J_{k+1} = \hat{J}^{\mu_{k+1}}.$
\STATE Set $k \leftarrow k+1.$ Go to 2.
\end{algorithmic}
\end{algorithm}

Before we present our main results, we first obtain bounds for modified policy iteration with lookahead and time-varying bounds in the policy evaluation error. The algorithm we analyze in this section is described in Algorithm \ref{alg:genalg2}. The algorithm in \cite{efroni2019combine} (Algorithm \ref{alg:genalg}) is similar to Algorithm \ref{alg:genalg2}, except at time $k$, the work of \cite{efroni2019combine} assumes a constant bound in the policy evaluation error, $\delta$, and in Algorithm \ref{alg:genalg}, we assume that the policy evaluation error is upper bounded by time-dependent $\delta_k.$ Then, we assume that $\delta_k$ is of the following form: $\delta_k \leq \beta^k \delta_0 + \mu$ when $0<\beta<1$. 
The bounds are given in Proposition \ref{proposition1}. In Section \ref{subsection3.4} we obtain values of $\beta$ and $\mu$ corresponding to Algorithm \ref{alg:LSalg}, approximate policy iteration with linear value function approximation and lookahead. We further extend the results to incorporate the use of gradient descent in Section \ref{SectionGD}. 

We now obtain a bound on the iterates in Algorithm \ref{alg:genalg2} as follows:
\begin{proposition} \label{proposition1}
\begin{align*}
    \norm{J^{\mu_{k}} - J^*}_\infty \nonumber&\leq  \frac{\alpha^{k(H)}}{1-\alpha} +\sum_{\ell=0}^{k-1}  \frac{\alpha^{(k-\ell-1)(H)}2\alpha^H \delta_\ell}{1-\alpha}+\frac{\eps_{LA}}{(1-\alpha)(1-\alpha^{H-1})}.
\end{align*}

Furthermore, when \begin{align}
   \delta_k \leq \beta^k \delta_0 + \mu \quad\text{ for } 0<\beta<1, 0<\mu.  
\end{align} then 
\begin{align} \norm{J^{\mu_{k}} - J^*}_\infty 
&\leq\underbrace{ \frac{\alpha^{k(H)}}{1-\alpha}  +  \frac{2\alpha^H}{1-\alpha}k \max(\alpha^{H-1},\beta)^{k-1}\delta_0}_{=: \textrm{ finite-time component}} +\underbrace{\frac{2\alpha^H\mu+\eps_{LA}}{(1-\alpha)(1-\alpha^{H})}}_{=: \textrm{ asympototic component}}.\nonumber
\end{align}

Taking limits on both sides, when $0<\beta<1,$ we have:
\begin{align*}
    \limsup_{k\to\infty }
    \norm{J^{\mu_{k}} - J^*}_\infty
    \leq\frac{2\alpha^H\mu+\eps_{LA}}{(1-\alpha)(1-\alpha^{H})}.\nonumber
\end{align*}
\end{proposition} 
\begin{remark}
We note that the results of \cite{efroni2019combine}, \cite{parr}, and \cite{bertsekas2019reinforcement} are all special cases of Proposition \ref{proposition1}.
\end{remark}

\subsection{Approximate Policy Iteration With Linear Value Function Approximation and Lookahead} \label{subsection3.4}

To apply Proposition~\ref{proposition1} to Algorithm~\ref{alg:LSalg}, we have compute the parameters $\beta$ and $\mu$ in the proposition. In Appendix \ref{appendix:AppendixC}, we show that $\beta$ and $\mu$ for Algorithm \ref{alg:LSalg} are given by
\begin{align}
\beta &:= \alpha^{m+H-1} \delta_{FV} \nonumber\\
\mu &:= \frac{\tau}{1-\beta}, \label{eq:betamu}
\end{align}
where $\tau := \frac{\alpha^m + \alpha^{m+H-1}}{1-\alpha}\delta_{FV} + \delta_{app}+ \delta_{FV} \eps_{PE}.$ 

Using \eqref{eq:betamu} along with Proposition \ref{proposition1}, we now state Theorem \ref{mainAPI}, which characterizes the role of lookahead ($H$) and return ($m$) 
on the convergence of approximate policy iteration with function approximation. 
\begin{theorem}\label{mainAPI}
Suppose that $m$ and $H$ satisfy $m + H -1>\log (\delta_{FV})/\log (1/\alpha),$ where 
$$\delta_{FV} := \sup_k \norm{\scriptM_k}_\infty = \sup_k\norm{\Phi(\Phi_{\scriptD_{k}}^\top \Phi_{\scriptD_{k}} )^{-1} \Phi_{\scriptD_{k}}^\top \scriptP_k}_\infty.$$
 Then, under Assumption \ref{assume 1 or}, the following holds for Algorithm \ref{alg:LSalg}:
\begin{align}
    \norm{J^{\mu_{k}} - J^*}_\infty 
    &\leq \underbrace{\frac{\alpha^{k(H)} }{1-\alpha} + \frac{2\alpha^{H}\norm{J^{\mu_0}-J_0}_\infty}{1-\alpha}k\max(\alpha^{H},\beta)^{k-1}}_{ \text{ finite-time component }} +\underbrace{\frac{2\alpha^H \frac{\tau}{1-\beta} +\eps_{LA}}{(1-\alpha^{H})(1-\alpha)}}_{ \text{ asymptotic component }}. 
    \label{eq:mainAPI bound}
\end{align}
 where  $$\tau:= \frac{\alpha^m + \alpha^{m+H-1}}{1-\alpha}\delta_{FV} + \delta_{app}+ \delta_{FV} \eps_{PE},$$
$$\beta:=\alpha^{m+H-1} \delta_{FV},$$
and
$$
\delta_{app} :=  \sup_{k, \mu_k}\norm{\scriptM_k J^{\mu_k}- J^{\mu_k}}_\infty.
$$
\end{theorem}

The proof of Theorem \ref{mainAPI} follows easily from Proposition \ref{proposition1}. In Appendix \ref{appendix:prop1} we give corresponding bounds on the iterates $J_k$ in the algorithm. We now make several comments about the implications of Theorem \ref{mainAPI}.

In light of the counterexample in Section \ref{subsection:counterexAppendix}, we note that the above result is fundamentally different from the conclusion in Theorem 4 in \cite{efroni2019combine} where one uses $J_k$ instead of using $T^{H-1}J_k$ in Step 2 of the algorithm. Here, we have shown that even when one uses $T^{H-1}J_k,$ one may need large $m+H$ for convergence due to the use of function approximation. The bounds in \cite{efroni2019combine} as well as \cite{bertsekas2019reinforcement},\cite{bucsoniu2012least},\cite{lagoudakis2001model} and \cite{parr} assume approximate policy iteration may always be bounded, however our counterexample shows otherwise, and sheds light on the possible necessity of lookahead in approximate policy iteration with linear value function approximation. 

Theorem \ref{mainAPI} can also be used to make the following observation: how close $J^{\mu_k}$ is to $J^*$ depends on four factors -- the representation power of the feature vectors and the feature vectors themselves ($\delta_{app}, \delta_{FV}$), the amount of lookahead ($H$), the extent of the rollout ($m$) and the approximation in the policy determination and policy evaluation steps ($\eps_{LA}$ and $\eps_{PE}$).  Theorem \ref{mainAPI} shows that while $\norm{J^{\mu_k}-J^*}_\infty$ depends on the function approximation error ($\delta_{app}$) and the feature vectors ($\delta_{FV}$), the effect of these terms diminishes exponentially with increased $H$, with the exception of the tree search error ($\eps_{LA}$). Further, it is easy to see that lookahead and rollout help mitigate the effect of feature vectors and their ability to represent the value functions. 

In \cite{bertsekas2021lessons}, it is noted that, in reinforcement learning to play computer games or board games, it is not uncommon during training to get a relatively crude estimate of the value function, which is improved by lookahead and $m$-step return during actual game play. Our analysis would also apply to this situation -- we have not explicitly differentiated between training and game play in our analysis.
\color{black}

\section{Extension To Gradient Descent} \label{SectionGD}
\begin{algorithm}[tb]
\caption{Gradient Descent Algorithm}
\label{alg:GDalg}
\textbf{Input}: $\theta_0, m,$ $H,$ feature vectors $\{ \phi(i) \}_{i \in \scriptS}, \phi(i) \in \mathbb{R}^d,$  and $\scriptD_k,$ which is the set of states for which we evaluate the current policy at iteration $k.$
\begin{algorithmic}[1] 
\STATE $k=0, J_0 = \Phi \theta_0$. 
\STATE Let $\mu_{k+1}$ be such that $\norm{T^H J_k - T_{\mu_{k+1}} T^{H-1} J_k}_\infty \leq \eps_{LA}$. 
\STATE Compute $\hat{J}^{\mu_{k+1}}(i) = T_{\mu_{k+1}}^m T^{H-1} J_k(i)+w_{k+1}(i)$ for $i \in \scriptD_k.$ 
\STATE \label{step 4 unbiased} $\theta_{k+1, 0} := \theta_k.$ For $\ell = 1, 2, \ldots, \eta,$ iteratively compute the following:
\begin{align}
    \theta_{k+1, \ell} &= \theta_{k+1,\ell-1} - \gamma  \nabla_\theta c(\theta;\hat{J}^{\mu_{k+1}})|_{\theta_{k+1,\ell-1}}, \label{eq:iterthetaGD}
\end{align} where
\begin{align*}
    c(\theta;\hat{J}^{\mu_{k+1}}) := \frac{1}{2}\sum_{i \in \scriptD } \Big( (\Phi \theta)(i) - \hat{J}^{\mu_{k+1}}(i) \Big)^2,
\end{align*}  \\
and $\Phi$ is a matrix whose rows are the feature vectors.\\
\STATE Define 
\begin{align*}
    \theta_{k+1} &= \theta_{k+1,\eta},
\end{align*} and set $$J_{k+1} = \Phi \theta_{k+1}.$$
\STATE Set $k \leftarrow k+1.$ Go to 2.
\end{algorithmic}
\end{algorithm}

Solving the least-squares problem in Algorithm~\ref{alg:LSalg} involves a matrix inversion, which can be computationally difficult. So, often this step is replaced by few steps of gradient performed on the least-squares objective. Here, we assume that we perform $\eta$ steps of gradient descent with stepsize $\gamma$ at each iteration $k$, where the gradient is the gradient of the least-squares objective in (\ref{step 4 alg}).

The gradient descent-based algorithm is presented in Algorithm~\ref{alg:GDalg}.
When $\gamma$ is sufficiently small and $\eta$ is sufficiently large, we have convergence to an asymptotic error, assuming that $m$ and $H$ are sufficiently large.
When we increase $\eta$, our asymptotic error becomes smaller until it reaches the asymptotic error of the least-squares algorithm, i.e., when $\eta \rightarrow \infty$, we recover the asymptotic error of Algorithm \ref{alg:LSalg}.

To apply Proposition~\ref{proposition1} to Algorithm~\ref{alg:GDalg}, we have to first identify the parameters $\beta$ and $\mu$ for this algorithm. We make the following assumption:
\begin{assumption} \label{assumption2}
$\gamma, m,\eta$ and $H$ satisfy
\begin{align*}
\gamma < \frac{1}{d\inf_k \norm{\Phi_{\scriptD_k}^\top \Phi_{\scriptD_k}}_\infty^2}, 
\end{align*}
 
\begin{align*}
m + H >1+\log (2\delta_{FV})/\log (1/\alpha)
\end{align*}
and  $$\eta>\log (\frac {3\sqrt{|S|}\norm{\Phi}_\infty}{\sigma_{\min,\Phi}})/\log (1/\alpha_{GD, \gamma}),$$  
where  $\alpha_{GD, \gamma}:=\sup_k \max_i |1-\gamma\lambda_i ( \Phi_{\scriptD_k}^\top \Phi_{\scriptD_k})|,$ where $\lambda_i$ denotes the $i$-th largest eigenvalue of a matrix and  $\sigma_{\min, \Phi}$ is the smallest singular value in the singular value decomposition of $\Phi$.
\end{assumption} 

Under Assumption \ref{assumption2}, we can obtain $\beta$ and $\mu$ for Algorithm \ref{alg:GDalg}. In Appendix \ref{appendix:AppendixD}, we show that $\beta$ and $\mu$ are given by
\begin{align}
    \nonumber&\beta = \alpha^{m+H-1} \delta_{FV}+ \frac{\sqrt{|S|}\norm{\Phi}_\infty}{\sigma_{\min,\Phi}}  \alpha_{GD, \gamma}^{\eta}( \alpha^{m+H-1} \delta_{FV}+1), \\
&\mu = \frac{\tau}{1-\beta} \label{eq:***}
\end{align} where $\tau := (1+ \frac{\sqrt{|S|}\norm{\Phi}_\infty}{\sigma_{\min,\Phi}}  \alpha_{GD, \gamma}^{\eta})( \frac{\alpha^m + \alpha^{m+H-1}}{1-\alpha}\delta_{FV} + \delta_{app}+ \delta_{FV} \eps_{PE}) +\frac{\sqrt{|S|}\norm{\Phi}_\infty}{(1-\alpha)\sigma_{\min,\Phi}}  \alpha_{GD, \gamma}^{\eta}.$

Using \eqref{eq:***} along with Proposition \ref{proposition1}, we now state our Theorem which characterizes the error in using gradient descent in approximate policy iteration with linear value function approximation and lookahead.
\begin{theorem}\label{mainAPIGD}
Suppose that $m$ and $H$ satisfy $m + H -1>\log (2\delta_{FV})/\log (1/\alpha),$ where 
$$\delta_{FV} := \sup_k \norm{\scriptM_k}_\infty = \sup_k\norm{\Phi(\Phi_{\scriptD_{k}}^\top \Phi_{\scriptD_{k}} )^{-1} \Phi_{\scriptD_{k}}^\top \scriptP_k}_\infty.$$
 Then, under Assumptions \ref{assume 1 or}-\ref{assumption2}, the following holds:
\begin{align}
    \norm{J^{\mu_{k}} - J^*}_\infty 
    &\leq \underbrace{\frac{\alpha^{kH} }{1-\alpha} + \frac{2\alpha^{H}\norm{J^{\mu_0}-J_0}_\infty}{1-\alpha}k\max(\alpha^{H},\beta)^{k-1}}_{ \text{ finite-time component }} +\underbrace{\frac{2\alpha^H \frac{\tau}{1-\beta} +\eps_{LA}}{(1-\alpha^{H})(1-\alpha)}}_{ \text{ asymptotic component }}. 
\end{align}
 where  $$\tau:= (1+ \frac{\sqrt{|S|}\norm{\Phi}_\infty}{\sigma_{\min,\Phi}}  \alpha_{GD, \gamma}^{\eta})( \frac{\alpha^m + \alpha^{m+H-1}}{1-\alpha}\delta_{FV} + \delta_{app}+ \delta_{FV} \eps_{PE}) +\frac{\sqrt{|S|}\norm{\Phi}_\infty}{(1-\alpha)\sigma_{\min,\Phi}}  \alpha_{GD, \gamma}^{\eta},$$
$$\beta:= \alpha^{m+H-1} \delta_{FV}+ \frac{\sqrt{|S|}\norm{\Phi}_\infty}{\sigma_{\min,\Phi}}  \alpha_{GD, \gamma}^{\eta}( \alpha^{m+H-1} \delta_{FV}+1),$$
and
$$
\delta_{app} :=  \sup_{k, \mu_k}\norm{\scriptM_k J^{\mu_k}- J^{\mu_k}}_\infty.
$$
\end{theorem}
Theorem \ref{mainAPIGD} follows directly from Proposition \ref{proposition1} when $\beta$ and $\mu$ are defined in \eqref{eq:***}.
\begin{remark}
Note that as $\eta \to \infty,$ i.e., the number of steps of gradient descent becomes very large, the error becomes the same as that of Algorithm \ref{alg:LSalg}.
\end{remark}

\section{Conclusion}

Practical RL algorithms that deal with large state spaces implement some form of approximate policy iteration. In traditional analyses of approximate policy iteration, for example in \cite{bertsekas2019reinforcement}, it is assumed that there is an error in the policy evaluation step and an error in the policy improvement step. The work of \cite{efroni2019combine} extends this analysis to incorporate lookahead policies, which mitigate the effects of function approximation. We provide a counterexample to show that incorporating linear value function into approximate policy iteration is not straightforward as the iterates may diverge. In this paper, we seek to understand the role of linear value function approximation in the policy evaluation step and the associated changes that one has to make to the approximate policy iteration algorithm (such as lookahead) to counteract the effect of function approximation.  Our main conclusion is that lookahead mitigates the effects of function approximation, rollout and the choice of specific feature vectors.

Possible directions for future work include the following:
\begin{itemize}
    \item In game playing applications, gradient descent is commonly used to estimate the value function, but  temporal-difference learning is used in other applications. It would be interesting to extend our results to the case of TD learning-based policy evaluation.
    \item While neural networks are not linear function approximators, recent results on the NTK analysis of neural networks suggest that they can be approximated as linear combinations of basis functions \cite{jacot2018neural,du2018gradient,arora2019fine,ji2019polylogarithmic, cao2019generalization}. Thus, to the extent that the NTK approximation is reasonable, our results can potentially shed light on why the combination of the representation capability of neural networks and tree-search methods work well in practice, although further work is necessary to make this connection precise.
\end{itemize} 

 \begin{APPENDICES}
\section{Proof of Proposition 1}
\label{appendix:AppendixB}
The work of \cite{efroni2019combine} shows that:
\begin{align}
    \norm{J^{\mu_{k+1}} - J^*}_\infty &\leq \alpha^{H} \norm{J^{\mu_k}-J^*}_\infty  +  \frac{  2\alpha^H \delta+ \eps_{LA} }{1-\alpha}.\label{eq: ** 4}
\end{align}

Iterating over $k$,
\begin{align*}
    \limsup_{k\to \infty} \norm{J^{\mu_k}-J^*}_\infty \leq \frac{2\alpha^H \delta + \eps_{LA}}{(1-\alpha)(1-\alpha^H)},
\end{align*}which is a main result of \cite{efroni2019combine}.
Suppose now that $\delta$ depends on $k$, and we call the sequence $\delta_k.$

Starting from \eqref{eq: ** 4}, we substitute $\delta_k$ for $\delta$ and we get the following:
\begin{align}
    \norm{J^{\mu_{k+1}} - J^*}_\infty &\leq \alpha^{H} \norm{J^{\mu_k}-J^*}_\infty  +  \frac{  2\alpha^H \delta_k+ \eps_{LA} }{1-\alpha}.\label{eq: mod 2 part 1}
\end{align} 

Iterating over $k$, we have:
\begin{align}
    \norm{J^{\mu_{k}} - J^*}_\infty \nonumber&\leq \alpha^{kH} \norm{J^{\mu_0}-J^*}_\infty  +  \frac{2\alpha^H}{1-\alpha}\sum_{\ell = 0}^{k-1}\frac{\alpha^{(k-\ell-1)(H-1)}2\alpha^H \delta_\ell +\eps_{LA} }{1-\alpha}\\
    &\leq \frac{\alpha^{kH}}{1-\alpha} +\sum_{\ell=0}^{k-1}  \frac{\alpha^{(k-\ell-1)(H)}2\alpha^H \delta_\ell+\eps_{LA}}{1-\alpha} \nonumber\\&\leq  \frac{\alpha^{kH}}{1-\alpha} +\sum_{\ell=0}^{k-1}  \frac{\alpha^{(k-\ell-1)(H)}2\alpha^H \delta_\ell}{1-\alpha}+\frac{\eps_{LA}}{(1-\alpha)(1-\alpha^{H-1})}.\label{eq: mod 2}
\end{align}

Note that for the bound in \eqref{eq: mod 2} to be useful, we need for the $\delta_k$ sequence to exhibit some properties that ensure the second term does not go to infinity as $k \to\infty.$

The bound in \eqref{eq: mod 2} can be further simplified if 
\begin{align}
   \delta_k \leq \beta^k \delta_0 + \mu \quad\text{ for } 0<\beta<1, 0<\mu.  \label{eq: mod 2 part 2}
\end{align}

Starting from \eqref{eq: mod 2 part 1}, where  $\delta_k= \beta^k \delta_0 +\mu $, we get the following:
\begin{align} \norm{J^{\mu_{k}} - J^*}_\infty 
&\leq \frac{\alpha^{k(H)}}{1-\alpha}  +  \frac{2\alpha^H}{1-\alpha}\sum_{\ell = 0}^{k-1}\alpha^{(k-\ell-1)(H)}  \Big[ \beta^\ell \delta_0 + \mu  \Big] + \frac{\eps_{LA}}{(1-\alpha)(1-\alpha^{H})} \nonumber\\
&\leq \frac{\alpha^{k(H)}}{1-\alpha}  +  \frac{2\alpha^H}{1-\alpha}\delta_0 \sum_{\ell = 0}^{k-1}\alpha^{(k-\ell-1)(H-1)}   \beta^\ell \nonumber+ \frac{2\alpha^H\mu+\eps_{LA}}{(1-\alpha)(1-\alpha^{H})} \nonumber\\
&\leq \frac{\alpha^{k(H)}}{1-\alpha} +  \frac{2\alpha^H}{1-\alpha}\delta_0 \sum_{\ell = 0}^{k-1}\max(\alpha^{H-1},\beta)^{k-1} \nonumber+ \frac{2\alpha^H\mu+\eps_{LA}}{(1-\alpha)(1-\alpha^{H})} \nonumber\\
&= \frac{\alpha^{k(H)}}{1-\alpha}  +  \frac{2\alpha^H}{1-\alpha}k \max(\alpha^{H-1},\beta)^{k-1}\delta_0 +\frac{2\alpha^H\mu+\eps_{LA}}{(1-\alpha)(1-\alpha^{H})}.\nonumber
\end{align}

Taking limits on both sides, noting that $0<\beta<1,$ we have:
\begin{align*}
    \limsup_{k\to\infty }
    \norm{J^{\mu_{k}} - J^*}_\infty
    \leq\frac{2\alpha^H\mu+\eps_{LA}}{(1-\alpha)(1-\alpha^{H})}.\nonumber
\end{align*}
\section{Obtaining $\beta$ and $\mu$ For Algorithm 1}
\label{appendix:AppendixC}
Using Assumption~\ref{assume 1 or}, $J_{k+1}$ can be written as
\begin{align*}
J_{k+1} &= \Phi \theta_{k+1} =\underbrace{\Phi(\Phi_{\scriptD_{k}}^\top \Phi_{\scriptD_{k}} )^{-1} \Phi_{\scriptD_{k}}^\top \scriptP_k}_{=: \scriptM_{k+1}} \hat{J}^{\mu_{k+1}}
\end{align*} 
where $\Phi_{\scriptD_{k}}$ is a matrix whose rows are the feature vectors of the states in $\scriptD_{k}$ and $\scriptP_k$ is a matrix of zeros and ones such that $\scriptP_k\hat{J}^{\mu_{k+1}}$ is a vector whose elements are a subset of the elements of $\hat{J}^{\mu_{k+1}}$ corresponding to $\scriptD_k$. Note that $\hat{J}^{\mu_{k+1}}(i)$ for $i\notin\scriptD_k$ does not affect the algorithm, so we can define $\hat{J}^{\mu_{k+1}}(i)=T_{\mu_{k+1}}^m T^{H-1} J_k(i)$ for $i\notin\scriptD_k.$

Written concisely, our algorithm is as follows:
\begin{equation}
    J_{k+1} = \scriptM_{k+1} (T_{\mu_{k+1}}^m T^{H-1} J_k+w_k), \end{equation}
where $\mu_{k+1}$ is defined in Step 2 of the algorithm. 
Since $w_k(i)$ for $i\notin\scriptD_k$ does not affect the algorithm, we define $w_k(i)=0$ for $i\notin\scriptD_k.$

Using contraction properties of $T_{\mu_k}$ and $T$, we obtain $\delta_k$ as follows: \begin{align*}
    \norm{J_k- J^{\mu_k}}_\infty&=\norm{\scriptM_k (T_{\mu_k}^m T^{H-1} J_{k-1}+w_k)- J^{\mu_k}}_\infty \\
    &\leq \norm{\scriptM_k T_{\mu_k}^m T^{H-1} J_{k-1}- J^{\mu_k}}_\infty + \norm{\scriptM_k w_k}_\infty \\
    &\leq\norm{\scriptM_k T_{\mu_k}^m T^{H-1} J_{k-1}- J^{\mu_k}}_\infty + \norm{\scriptM_k}_\infty \norm{w_k}_\infty \\ \allowdisplaybreaks
    &\leq \norm{\scriptM_k T_{\mu_k}^m T^{H-1} J_{k-1}- J^{\mu_k}}_\infty + \delta_{FV} \eps_{PE} \\ \allowdisplaybreaks
    &=  \norm{\scriptM_k T_{\mu_k}^m T^{H-1} J_{k-1} - \scriptM_k J^{\mu_k} + \scriptM_k J^{\mu_k}- J^{\mu_k}}_\infty + \delta_{FV} \eps_{PE}\\
    &\leq \norm{\scriptM_k T_{\mu_k}^m T^{H-1} J_{k-1} - \scriptM_k J^{\mu_k}}_\infty + \norm{\scriptM_k J^{\mu_k}- J^{\mu_k}}_\infty+ \delta_{FV} \eps_{PE}\\ 
    &\leq \norm{\scriptM_k}_\infty \norm{ T_{\mu_k}^m T^{H-1} J_{k-1} - J^{\mu_k}}_\infty + \norm{\scriptM_k J^{\mu_k}- J^{\mu_k}}_\infty+ \delta_{FV} \eps_{PE}\\ 
    &\leq \alpha^m \norm{\scriptM_k}_\infty \norm{T^{H-1} J_{k-1} - J^{\mu_k}}_\infty + \sup_{k, \mu_k}\norm{\scriptM_k J^{\mu_k}- J^{\mu_k}}_\infty + \delta_{FV} \eps_{PE}\\
    &\leq  \alpha^m\norm{\scriptM_k}_\infty \norm{T^{H-1} J_{k-1} - J^* + J^* - J^{\mu_k}}_\infty +\delta_{app} + \delta_{FV} \eps_{PE}\\
    &\leq  \alpha^m\norm{\scriptM_k}_\infty \norm{T^{H-1} J_{k-1} - J^*}_\infty + \alpha^m\norm{\scriptM_k}_\infty \norm{J^* - J^{\mu_k}}_\infty + \delta_{app}+ \delta_{FV} \eps_{PE} \\
    &\leq  \alpha^{m+H-1}\norm{\scriptM_k}_\infty \norm{J_{k-1} - J^*}_\infty + \frac{\alpha^m}{1-\alpha}\norm{\scriptM_k}_\infty +\delta_{app} + \delta_{FV} \eps_{PE}\\
    &\leq  \alpha^{m+H-1}\norm{\scriptM_k}_\infty \norm{J_{k-1} -J^{\mu_{k-1}} + J^{\mu_{k-1}}- J^*}_\infty + \frac{\alpha^m}{1-\alpha}\norm{\scriptM_k}_\infty +\delta_{app} + \delta_{FV} \eps_{PE}\\
    &\leq  \alpha^{m+H-1}\norm{\scriptM_k}_\infty \norm{J_{k-1} -J^{\mu_{k-1}}}_\infty + \frac{\alpha^m + \alpha^{m+H-1}}{1-\alpha}\norm{\scriptM_k}_\infty +\delta_{app} + \delta_{FV} \eps_{PE}\\
    &\leq \alpha^{m+H-1} \delta_{FV} \norm{J_{k-1} -J^{\mu_{k-1}}}_\infty + \frac{\alpha^m + \alpha^{m+H-1}}{1-\alpha}\delta_{FV} + \delta_{app}+ \delta_{FV} \eps_{PE}. 
\end{align*} 

Now we have:

\begin{align*}
\underbrace{\norm{J_k- J^{\mu_k}}_\infty}_{\delta_k}
&\leq \underbrace{\alpha^{m+H-1} \delta_{FV}}_{=:\beta} \underbrace{\norm{J_{k-1} -J^{\mu_{k-1}}}_\infty}_{\delta_{k-1}} + \underbrace{\frac{\alpha^m + \alpha^{m+H-1}}{1-\alpha}\delta_{FV} + \delta_{app}+ \delta_{FV} \eps_{PE}}_{=:\tau}. 
\end{align*} 
Iterating, 
\begin{align}
     \delta_k &\leq \beta^k \delta_0 + \sum_{i=0}^{k-1}\beta^i \tau\nonumber \\
     &\leq \beta^k \delta_0 + \underbrace{\frac{\tau}{1-\beta}}_{=: \mu}. \label{eq: modification 2}
\end{align}

\section{A Modified Least Squares Algorithm}\label{appendix:thm2}

Suppose Step \ref{step 3 alg} of Algorithm \ref{alg:LSalg} is changed to  $\hat{J}^{\mu_{k+1}}(i) = T_{\mu_{k+1}}^m  (J_k)(i)+w_{k+1}(i)$ for $i \in \scriptD_k$. Then, it is still possible to get bounds on the performance of the algorithm when $m$ is sufficiently large. With this modification to the algorithm, we have the following:
\begin{proposition}\label{mainAPISecond}
Suppose that $m$ satisfies $m>\log (\delta_{FV})/\log (1/\alpha),$ where 
$$\delta_{FV} := \sup_k \norm{\scriptM_k}_\infty = \sup_k\norm{\Phi(\Phi_{\scriptD_{k}}^\top \Phi_{\scriptD_{k}} )^{-1} \Phi_{\scriptD_{k}}^\top \scriptP_k}_\infty.$$
 Then, under Assumption \ref{assume 1 or}, the following holds:
\begin{align*}
    \norm{J^{\mu_{k}} - J^*}_\infty 
    &\leq \underbrace{\frac{\alpha^{k(H)} }{1-\alpha} + \frac{2\alpha^{H}\norm{J^{\mu_0}-J_0}_\infty}{1-\alpha}k\max(\alpha^{H},\beta')^{k-1}}_{ \text{ finite-time component }} +\underbrace{\frac{2\alpha^H \frac{\tau'}{1-\beta'} +\eps_{LA}}{(1-\alpha^{H})(1-\alpha)}}_{ \text{ asymptotic component }}. 
\end{align*}
 where  $$\tau':= \alpha^m \delta_{FV},$$
$$\beta':=\frac{\alpha^m \delta_{FV}}{1-\alpha} + \delta_{app} + \delta_{FV} \eps_{PE},$$
and
$$
\delta_{app} :=  \sup_{k, \mu_k}\norm{\scriptM_k J^{\mu_k}- J^{\mu_k}}_\infty.
$$
\end{proposition}
\begin{proof}{Proof of Proposition \ref{mainAPISecond}}
The proof of Theorem \ref{mainAPISecond} is similar to the proof of Theorem \ref{mainAPI}. We thus give the following iteration which can be substituted in our proof of Theorem \ref{mainAPI}:
\begin{align*}
    \norm{J_k-J^{\mu_k}}_\infty &=\norm{\scriptM_k (T_{\mu_k}^m  J_{k-1}+w_k)- J^{\mu_k}}_\infty \\&= \norm{\scriptM_k (T_{\mu_k}^m  J_{k-1}+w_k)- J^{\mu_k}}_\infty \\
    &\leq \norm{\scriptM_k T_{\mu_k}^m  J_{k-1}- J^{\mu_k}}_\infty + \norm{\scriptM_k w_k}_\infty \\
    &\leq\norm{\scriptM_k T_{\mu_k}^m  J_{k-1}- J^{\mu_k}}_\infty + \norm{\scriptM_k}_\infty \norm{w_k}_\infty \\ \allowdisplaybreaks
    &\leq \norm{\scriptM_k T_{\mu_k}^m  J_{k-1}- J^{\mu_k}}_\infty + \delta_{FV} \eps_{PE} \\ \allowdisplaybreaks
    &=  \norm{\scriptM_k T_{\mu_k}^m  J_{k-1} - \scriptM_k J^{\mu_k} + \scriptM_k J^{\mu_k}- J^{\mu_k}}_\infty + \delta_{FV} \eps_{PE}\\
    &\leq \norm{\scriptM_k T_{\mu_k}^m  J_{k-1} - \scriptM_k J^{\mu_k}}_\infty + \norm{\scriptM_k J^{\mu_k}- J^{\mu_k}}_\infty+ \delta_{FV} \eps_{PE}\\
    &\leq \sup_k \norm{\scriptM_k}_\infty \norm{ T_{\mu_k}^m  J_{k-1} - J^{\mu_k}}_\infty + \sup_{k, \mu_k}\norm{\scriptM_k J^{\mu_k}- J^{\mu_k}}_\infty+ \delta_{FV} \eps_{PE}\\
    &\leq \alpha^m \delta_{FV} \norm{ J_{k-1} - J^{\mu_k}}_\infty + \delta_{app} + \delta_{FV} \eps_{PE}\\
    &=\alpha^m \delta_{FV} \norm{ J_{k-1} -J^{\mu_{k-1}}+J^{\mu_{k-1}}- J^{\mu_k}}_\infty + \delta_{app} + \delta_{FV} \eps_{PE}\\
    &\leq \alpha^m \delta_{FV} \norm{ J_{k-1} -J^{\mu_{k-1}}}_\infty+\alpha^m \delta_{FV}\norm{J^{\mu_{k-1}}- J^{\mu_k}}_\infty + \delta_{app} + \delta_{FV} \eps_{PE}\\
    &\leq \alpha^m \delta_{FV} \norm{ J_{k-1} -J^{\mu_{k-1}}}_\infty+\frac{\alpha^m \delta_{FV}}{1-\alpha} + \delta_{app} + \delta_{FV} \eps_{PE}.
\end{align*}
Substituting $$\beta':= \alpha^m \delta_{FV}$$ and $$\tau' :=\frac{\alpha^m \delta_{FV}}{1-\alpha} + \delta_{app} + \delta_{FV} \eps_{PE},$$ in place of $\beta$ and $\tau$, respectively, in the proof of Theorem \ref{mainAPI}, we obtain Proposition \ref{mainAPISecond}.
\end{proof}

\section{Bounds on $J_k$ In Algorithm \ref{alg:LSalg}}\label{appendix:prop1} 
In the following proposition, we present a bound on the difference between $J_k$ and $J^*.$ 

\begin{proposition} \label{IterAPITheorem}When $\alpha^{m+H-1} \delta_{FV} <1,$
\begin{align*}
\limsup_{k\to\infty} \norm{J_{k} - J^*}_\infty &\leq  \frac{ \big( 1+\delta_{FV} \alpha^m \big) \Big[\frac{2\alpha^H \frac{\tau}{1-\beta} +\eps_{LA}}{(1-\alpha^{H})(1-\alpha)} \Big] + \delta_{app}+\delta_{FV} \eps_{LA}}{1-\delta_{FV} \alpha^{m+H-1}},
\end{align*}
where $\beta$ and $\tau$ are defined in Theorem \ref{mainAPI}.
\end{proposition}
The proof is as follows.
\begin{proof}{Proof of Proposition \ref{IterAPITheorem}}
\begin{align*}
\norm{J_{k+1} - J^*}_\infty &= \norm{J_{k+1} -J^{\mu_{k+1}}  + J^{\mu_{k+1}}- J^*}_\infty \\
&\leq \norm{J_{k+1} -J^{\mu_{k+1}}}_\infty + \norm{J^{\mu_{k+1}}- J^*}_\infty \\
&\leq \norm{\scriptM_{k+1} T_{\mu_{k+1}}^m T^{H-1} J_k -J^{\mu_{k+1}}}_\infty +\delta_{FV} \eps_{LA} \\&+\norm{J^{\mu_{k+1}}- J^*}_\infty +\delta_{FV} \eps_{LA}\\
&= \norm{\scriptM_{k+1} T_{\mu_{k+1}}^m T^{H-1} J_k -\scriptM_{k+1} J^{\mu_{k+1}} + \scriptM_{k+1} J^{\mu_{k+1}} -J^{\mu_{k+1}}}_\infty  \\&+\norm{J^{\mu_{k+1}}- J^*}_\infty +\delta_{FV} \eps_{LA}\\
&\leq \norm{\scriptM_{k+1} T_{\mu_{k+1}}^m T^{H-1} J_k -\scriptM_{k+1} J^{\mu_{k+1}}}_\infty + \norm{\scriptM_{k+1} J^{\mu_{k+1}} -J^{\mu_{k+1}}}_\infty  \\&+\norm{J^{\mu_{k+1}}- J^*}_\infty +\delta_{FV} \eps_{LA}\\
&\leq \norm{\scriptM_{k+1}}_\infty \norm{T_{\mu_{k+1}}^m T^{H-1} J_k - J^{\mu_{k+1}}}_\infty + \norm{\scriptM_{k+1} J^{\mu_{k+1}} -J^{\mu_{k+1}}}_\infty  \\&+\norm{J^{\mu_{k+1}}- J^*}_\infty +\delta_{FV} \eps_{LA}\\
&\leq \delta_{FV} \alpha^m \norm{T^{H-1} J_k - J^{\mu_{k+1}}}_\infty + \delta_{app} \\&+\norm{J^{\mu_{k+1}}- J^*}_\infty +\delta_{FV} \eps_{LA}\\
&= \delta_{FV} \alpha^m \norm{T^{H-1} J_k - J^* + J^* - J^{\mu_{k+1}}}_\infty + \delta_{app} +\norm{J^{\mu_{k+1}}- J^*}_\infty +\delta_{FV} \eps_{LA}\\
&\leq \delta_{FV} \alpha^m \norm{T^{H-1} J_k - J^*}_\infty + \delta_{FV} \alpha^m  \norm{J^* - J^{\mu_{k+1}}}_\infty + \delta_{app} \\&+ \norm{J^{\mu_{k+1}}- J^*}_\infty+\delta_{FV} \eps_{LA}\\
&\leq \delta_{FV} \alpha^{m+H-1} \norm{J_k - J^*}_\infty + \delta_{FV} \alpha^m  \norm{J^* - J^{\mu_{k+1}}}_\infty + \delta_{app} \\&+ \norm{J^{\mu_{k+1}}- J^*}_\infty+\delta_{FV} \eps_{LA}\\
&= \delta_{FV} \alpha^{m+H-1} \norm{J_k - J^*}_\infty + \big( 1+\delta_{FV} \alpha^m \big) \norm{J^* - J^{\mu_{k+1}}}_\infty + \delta_{app}+\delta_{FV} \eps_{LA}.
\end{align*}
From Theorem \ref{mainAPI}, we have that 
\begin{align*}
    \limsup_{k\to\infty} \norm{J^{\mu_k}-J^*}_\infty \leq \frac{2\alpha^H \frac{\tau}{1-\beta} +\eps_{LA}}{(1-\alpha^{H})(1-\alpha)}.
\end{align*}

Thus, for every $\eps'>0,$ there exists a $k(\eps')$ such that for all $k>k(\eps')$, 
\begin{align*}
     \norm{J^{\mu_k}-J^*}_\infty \leq \frac{2\alpha^H \frac{\tau}{1-\beta} +\eps_{LA}}{(1-\alpha^{H})(1-\alpha)}+\eps'.
\end{align*}
Thus, for all $k>k(\eps')$, we have:
\begin{align*}
\norm{J_{k+1} - J^*}_\infty &\leq \delta_{FV} \alpha^{m+H-1} \norm{J_k - J^*}_\infty + \big( 1+\delta_{FV} \alpha^m \big) \Big[\frac{2\alpha^H \frac{\tau}{1-\beta} +\eps_{LA}}{(1-\alpha^{H})(1-\alpha)}+\eps' \Big] + \delta_{app}+\delta_{FV} \eps_{LA}.
\end{align*}
Iterating over $k$ gives us:
\begin{align*}
\limsup_{k\to\infty} \norm{J_{k} - J^*}_\infty &\leq  \frac{ \big( 1+\delta_{FV} \alpha^m \big) \Big[\frac{2\alpha^H \frac{\tau}{1-\beta} +\eps_{LA}}{(1-\alpha^{H})(1-\alpha)}+\eps' \Big] + \delta_{app}+\delta_{FV} \eps_{LA}}{1-\delta_{FV} \alpha^{m+H-1}}.
\end{align*}
Since the above holds for all $\eps'$:
\begin{align*}
\limsup_{k\to\infty} \norm{J_{k} - J^*}_\infty &\leq  \frac{ \big( 1+\delta_{FV} \alpha^m \big) \Big[\frac{2\alpha^H \frac{\tau}{1-\beta} +\eps_{LA}}{(1-\alpha^{H})(1-\alpha)} \Big] + \delta_{app}+\delta_{FV} \eps_{LA}}{1-\delta_{FV} \alpha^{m+H-1}}.
\end{align*}

\end{proof}

\section{Obtaining $\beta$ and $\mu$ for Algorithm 2}
\label{appendix:AppendixD}

In order to derive $\beta$ and $\mu$ for Algorithm \ref{alg:GDalg}, we define $\tilde{\theta}^{\mu_k}$ for any policy $\mu_k$:
\begin{align}
    \tilde{\theta}^{\mu_k} \nonumber&:= \arg\min_\theta \frac{1}{2}\norm{\Phi_{\scriptD_k} \theta - \scriptP_{k}(T_{\mu_{k}}^m T^{H-1} J_{k-1}+w_k)}_2^2.
\end{align} 
Note that 
\begin{align}
    \Phi\tilde{\theta}^{\mu_k}= \scriptM_k (T_{\mu_k}^m T^{H-1} J_{k-1}+w_k), \label{eq: theta tilde muk or}
\end{align} 
where $\scriptM_k$ is defined in \eqref{defMk}.
Thus, $\tilde{\theta}^{\mu_k}$ represents the function approximation of the estimate of $J^{\mu_k}$ obtained from the $m$-step return.

First, since $\theta_{k}$ is obtained by taking $\eta$ steps of gradient descent towards $\tilde{\theta}^{\mu_{k}}$ beginning from $\theta_{k-1}$, we show that the following holds:
\begin{align*}
   \norm{\theta_{k} - \tilde{\theta}^{\mu_{k}}}_2 
   \leq \alpha_{GD, \gamma}^{\eta} \norm{\theta_{k-1} - \tilde{\theta}^{\mu_{k}}}_2,
\end{align*}
 where $\alpha_{GD, \gamma}:=\sup_k \max_i |1-\gamma\lambda_i ( \Phi_{\scriptD_k}^\top \Phi_{\scriptD_k})|,$ where $\lambda_i$ denotes the $i$-th largest eigenvalue of a matrix.

We note that since $$0 < \lambda_i ( \Phi_{\scriptD_k}^\top \Phi_{\scriptD_k}) \leq \norm{\Phi_{\scriptD_k}^\top \Phi_{\scriptD_k}}_2^2 \leq d  \norm{\Phi_{\scriptD_k}^\top \Phi_{\scriptD_k}}_\infty^2 \leq d  \sup_k \norm{\Phi_{\scriptD_k}^\top \Phi_{\scriptD_k}}_\infty^2,$$ 
$\alpha_{GD, \gamma}<1$ when $\gamma < \frac{1}{d\sup_k \norm{\Phi_{\scriptD_k}^\top \Phi_{\scriptD_k}}_\infty^2}$, which follows from Assumption \ref{assumption2}.

 Recall that the iterates in Equation \eqref{eq:iterthetaGD} can be written as follows:
\begin{align*}
\theta_{k,\ell} &= \theta_{k,\ell-1} - \gamma  \nabla_\theta c(\theta;\hat{J}^{\mu_{k}})|_{\theta_{k,\ell-1}} =\theta_{k,\ell-1} - \gamma \Big( \Phi_{\scriptD_{k-1}}^\top \Phi_{\scriptD_{k-1}} \theta_{k,\ell-1} - \Phi_{\scriptD_{k-1}}^\top \scriptP_{k-1}(T_{\mu_{k}}^m  T^{H-1}J_k+w_{k-1})\Big).
\end{align*}

Since 
\begin{align*}0 &= \nabla_\theta c(\theta;\hat{J}^{\mu_{k}})|_{\tilde{\theta}^{\mu_{k}}}= \Phi_{\scriptD_{k-1}}^\top \Phi_{\scriptD_{k-1}} \tilde{\theta}^{\mu_{k}} - \Phi_{\scriptD_{k-1}}^\top \scriptP_{k-1}(T_{\mu_{k}}^m  T^{H-1}J_k+w_{k-1}),
\end{align*}
we have the following:
\begin{equation*}
\begin{array}{lll}
\theta_{k,\ell} &=& \theta_{k,\ell-1} - \gamma \Big( \Phi_{\scriptD_{k-1}}^\top \Phi_{\scriptD_{k-1}} \theta_{k,\ell-1} - \Phi_{\scriptD_{k-1}}^\top \Phi_{\scriptD_{k-1}} \tilde{\theta}^{\mu_{k}} - \Phi_{\scriptD_{k-1}}^\top \scriptP_{k-1}(T_{\mu_{k}}^m  T^{H-1}J_k+w_{k-1}) \\&+& \Phi_{\scriptD_{k-1}}^\top \scriptP_{k-1}(T_{\mu_{k}}^m  T^{H-1}J_k+w_{k-1})\Big)    \\
&=& \theta_{k,\ell-1} - \gamma \Phi_{\scriptD_{k-1}}^\top \Phi_{\scriptD_{k-1}} (\theta_{k,\ell-1} - \tilde{\theta}^{\mu_{k}}).
\end{array}
\end{equation*}
Subtracting $\tilde{\theta}^{\mu_{k}}$ from both sides gives:
\begin{align*}
\theta_{k,\ell} - \tilde{\theta}^{\mu_{k}} &=  \theta_{k,\ell-1} - \tilde{\theta}^{\mu_{k}} - \gamma \Phi_{\scriptD_{k-1}}^\top \Phi_{\scriptD_{k-1}} (\theta_{k,\ell-1} - \tilde{\theta}^{\mu_{k}})\\&= (I - \gamma \Phi_{\scriptD_{k-1}}^\top \Phi_{\scriptD_{k-1}}) (\theta_{k,\ell-1} - \tilde{\theta}^{\mu_{k}}).
\end{align*}

Thus, 
\begin{align*}
 \norm{\theta_{k,\ell} - \tilde{\theta}^{\mu_{k}}}_2&=  \norm{(I - \gamma \Phi_{\scriptD_{k-1}}^\top \Phi_{\scriptD_{k-1}}) (\theta_{k,\ell-1} - \tilde{\theta}^{\mu_{k}})}_2\\&\leq  \norm{I - \gamma \Phi_{\scriptD_{k-1}}^\top \Phi_{\scriptD_{k-1}}}_2 \norm{\theta_{k,\ell-1} - \tilde{\theta}^{\mu_{k}}}_2\\
 &\leq \max_i |\lambda_i (I - \gamma \Phi_{\scriptD_{k-1}}^\top \Phi_{\scriptD_{k-1}})|\norm{\theta_{k,\ell-1} - \tilde{\theta}^{\mu_{k}}}_2\\
 &\leq \max_i |1-\gamma\lambda_i ( \Phi_{\scriptD_{k-1}}^\top \Phi_{\scriptD_{k-1}})|\norm{\theta_{k,\ell-1} - \tilde{\theta}^{\mu_{k}}}_2 \\
 &\leq \underbrace{\sup_k \max_i |1-\gamma\lambda_i ( \Phi_{\scriptD_k}^\top \Phi_{\scriptD_k})|}_{=: \alpha_{GD, \gamma}}\norm{\theta_{k,\ell-1} - \tilde{\theta}^{\mu_{k}}}_2,
\end{align*} where $\lambda_i$ denotes the $i$-th largest eigenvalue of a matrix.

Iterating over $k,$ the following holds:
\begin{align}
   \norm{\theta_{k} - \tilde{\theta}^{\mu_{k}}}_2 &=\norm{\theta_{k,\eta} - \tilde{\theta}^{\mu_{k}}}_2\nonumber\\
   &\leq \alpha_{GD, \gamma}^{\eta} \norm{\theta_{k,0} - \tilde{\theta}^{\mu_{k}}}_2\nonumber\\
   &=  \alpha_{GD, \gamma}^{\eta} \norm{\theta_{k-1} - \tilde{\theta}^{\mu_{k}}}_2\label{eq:matnorm}.
\end{align}


Using \eqref{eq:matnorm} as well as equivalence and sub-multiplicative properties of matrix norms, we have the following:
\begin{alignat*}{2}
    &\frac{1}{\norm{\Phi}_\infty} \norm{\Phi \theta_k-\Phi \tilde{\theta}^{\mu_{k}}}_\infty &&\leq \norm{\theta_k - \tilde{\theta}^{\mu_{k}}}_\infty 
    \\& &&\leq \norm{\theta_k - \tilde{\theta}^{\mu_{k}}}_2 
    \\& &&\leq \alpha_{GD, \gamma}^{\eta} \norm{\theta_{k-1} - \tilde{\theta}^{\mu_{k}}}_2
    \\& &&\leq \frac{1}{\sigma_{\min,\Phi}} \alpha_{GD, \gamma}^{\eta} \norm{\Phi\theta_{k-1} - \Phi\tilde{\theta}^{\mu_{k}}}_2 
    \\& &&\leq \frac{\sqrt{|S|}}{\sigma_{\min,\Phi}}  \alpha_{GD, \gamma}^{\eta} \norm{\Phi\theta_{k-1} - \Phi\tilde{\theta}^{\mu_{k}}}_\infty\\
    &\implies \norm{J_k-\Phi \tilde{\theta}^{\mu_{k}}}_\infty&&\leq \frac{\sqrt{|S|}\norm{\Phi}_\infty}{\sigma_{\min,\Phi}}  \alpha_{GD, \gamma}^{\eta} \norm{J_{k-1} - \Phi\tilde{\theta}^{\mu_{k}}}_\infty ,
\end{alignat*}
 where $\sigma_{\min, \Phi}$ is the smallest singular value in the singular value decomposition of $\Phi$ and the last line follows from the fact that $J_k := \Phi \theta_k.$

The above implies the following:
\begin{align}
    \norm{J^{\mu_{k}}-J_k}_\infty &\leq \norm{\Phi\tilde{\theta}^{\mu_{k}}-J^{\mu_k}}_\infty +\frac{\sqrt{|S|}\norm{\Phi}_\infty}{\sigma_{\min,\Phi}}  \alpha_{GD, \gamma}^{\eta}\norm{J_{k-1} - \Phi\tilde{\theta}^{\mu_{k}}}_\infty\nonumber\\
    &= \norm{\scriptM_k (T_{\mu_k}^m T^{H-1} J_{k-1}+w_k)-J^{\mu_k}}_\infty +\frac{\sqrt{|S|}\norm{\Phi}_\infty}{\sigma_{\min,\Phi}}  \alpha_{GD, \gamma}^{\eta}\norm{J_{k-1} - \Phi\tilde{\theta}^{\mu_{k}}}_\infty,\label{eq: label 1 or}
\end{align}
where the equality follows from \eqref{eq: theta tilde muk or}.

Now we bound $\norm{J_{k-1}-\Phi\tilde{\theta}^{\mu_{k}}}_\infty$ as follows:
\begin{align}
  \norm{J_{k-1}-\Phi\tilde{\theta}^{\mu_{k}}}_\infty 
  \nonumber&\leq \norm{J_{k-1} -  J^{\mu_{k-1}}}_\infty + \norm{ J^{\mu_{k-1}}-J^{\mu_{k}}}_\infty + \norm{J^{\mu_{k}}-\Phi\tilde{\theta}^{\mu_{k}}}_\infty \\
  \nonumber&\leq \norm{J_{k-1} -  J^{\mu_{k-1}}}_\infty + \frac{1}{1-\alpha} + \norm{J^{\mu_{k}}-\Phi\tilde{\theta}^{\mu_{k}}}_\infty
  \\
  &\leq \norm{J_{k-1} -  J^{\mu_{k-1}}}_\infty + \frac{1}{1-\alpha} + \norm{J^{\mu_{k}}-\scriptM_k (T_{\mu_k}^m T^{H-1} J_{k-1}+w_k)}_\infty, \label{eq: label 2 or}
\end{align}
where the last line follows from \eqref{eq: theta tilde muk or}. We use our upper bound on $\norm{\scriptM_k (T_{\mu_k}^m T^{H-1} J_{k-1}+w_k)- J^{\mu_k}}_\infty$ introduced in Appendix \ref{appendix:AppendixC} to put together with \eqref{eq: label 1 or} and \eqref{eq: label 2 or}, and get the following:
\begin{align*}
        \underbrace{\norm{J^{\mu_{k}}-J_k}_\infty}_{\delta_k} \leq 
      \beta \underbrace{\norm{J_{k-1} -J^{\mu_{k-1}}}_\infty}_{\delta_{k-1}} +\tau, 
\end{align*}
 \begin{align*}\beta := \alpha^{m+H-1} \delta_{FV}+ \frac{\sqrt{|S|}\norm{\Phi}_\infty}{\sigma_{\min,\Phi}}  \alpha_{GD, \gamma}^{\eta}( \alpha^{m+H-1} \delta_{FV}+1) \end{align*} and
\begin{align*}
\tau := (1+ \frac{\sqrt{|S|}\norm{\Phi}_\infty}{\sigma_{\min,\Phi}}  \alpha_{GD, \gamma}^{\eta})( \frac{\alpha^m + \alpha^{m+H-1}}{1-\alpha}\delta_{FV} + \delta_{app}+ \delta_{FV} \eps_{PE}) +\frac{\sqrt{|S|}\norm{\Phi}_\infty}{(1-\alpha)\sigma_{\min,\Phi}}  \alpha_{GD, \gamma}^{\eta}.
\end{align*}

Thus, we get:
\begin{align*}
    &\mu = \frac{\tau}{1-\beta}.
\end{align*} when $0<\beta<1,$ which follows from the assumptions in Proposition \ref{proposition1} and Assumption \ref{assumption2}.
\end{APPENDICES}
\ACKNOWLEDGMENT{The research presented here was supported in part by a grant from Sandia National Labs and the NSF Grants CCF 1934986, CCF 2207547, CNS 2106801, ONR Grant N00014-19-1-2566, and ARO Grant W911NF-19-1-0379. Sandia National Laboratories is a multimission laboratory managed and operated by National Technology \& Engineering Solutions of Sandia, LLC, a wholly owned subsidiary of Honeywell International Inc., for the U.S. Department of Energy’s National Nuclear Security Administration under contract DE-NA0003525.  This paper describes objective technical results and analysis. Any subjective views or opinions that might be expressed in the paper do not necessarily represent the views of the U.S. Department of Energy or the United States Government. 
}

\bibliography{arxiv2.bib}\bibliographystyle{informs2014}

\begin{thebibliography}{33}
\providecommand{\natexlab}[1]{#1}
\providecommand{\url}[1]{\texttt{#1}}
\providecommand{\urlprefix}{URL }

\bibitem[{Arora et~al.(2019)Arora, Du, Hu, Li, \protect\BIBand{}
  Wang}]{arora2019fine}
Arora S, Du S, Hu W, Li Z, Wang R (2019) Fine-grained analysis of optimization
  and generalization for overparameterized two-layer neural networks.
  \emph{International Conference on Machine Learning}, 322--332 (PMLR).

\bibitem[{Baxter et~al.(1999)Baxter, Tridgell, \protect\BIBand{}
  Weaver}]{baxter}
Baxter J, Tridgell A, Weaver L (1999) Tdleaf(lambda): Combining temporal
  difference learning with game-tree search. \emph{CoRR} cs.LG/9901001,
  \urlprefix\url{https://arxiv.org/abs/cs/9901001}.

\bibitem[{Bertsekas(2011)}]{Bertsekas2011ApproximatePI}
Bertsekas D (2011) Approximate policy iteration: a survey and some new methods.
  \emph{Journal of Control Theory and Applications} 9:310--335.

\bibitem[{Bertsekas(2021)}]{bertsekas2021lessons}
Bertsekas D (2021) Lessons from alphazero for optimal, model predictive, and
  adaptive control.

\bibitem[{Bertsekas \protect\BIBand{} Tsitsiklis(1996)}]{bertsekastsitsiklis}
Bertsekas D, Tsitsiklis J (1996) \emph{Neuro-dynamic Programming} (Athena
  Scientific), ISBN 9781886529106.

\bibitem[{Bertsekas(2019)}]{bertsekas2019reinforcement}
Bertsekas DP (2019) \emph{Reinforcement learning and optimal control} (Athena
  Scientific Belmont, MA).

\bibitem[{Browne et~al.(2012)Browne, Powley, Whitehouse, Lucas, Cowling,
  Rohlfshagen, Tavener, Perez~Liebana, Samothrakis, \protect\BIBand{}
  Colton}]{browne}
Browne C, Powley E, Whitehouse D, Lucas S, Cowling P, Rohlfshagen P, Tavener S,
  Perez~Liebana D, Samothrakis S, Colton S (2012) A survey of monte carlo tree
  search methods. \emph{IEEE Transactions on Computational Intelligence and AI
  in Games} 4:1:1--43,
  \urlprefix\url{http://dx.doi.org/10.1109/TCIAIG.2012.2186810}.

\bibitem[{Bu{\c{s}}oniu et~al.(2012)Bu{\c{s}}oniu, Lazaric, Ghavamzadeh, Munos,
  Babu{\v{s}}ka, \protect\BIBand{} Schutter}]{bucsoniu2012least}
Bu{\c{s}}oniu L, Lazaric A, Ghavamzadeh M, Munos R, Babu{\v{s}}ka R, Schutter
  BD (2012) Least-squares methods for policy iteration. \emph{Reinforcement
  learning} 75--109.

\bibitem[{Cao \protect\BIBand{} Gu(2019)}]{cao2019generalization}
Cao Y, Gu Q (2019) Generalization bounds of stochastic gradient descent for
  wide and deep neural networks. \emph{Advances in Neural Information
  Processing Systems} 32:10836--10846.

\bibitem[{Deng et~al.(2020)Deng, Yin, Deng, \protect\BIBand{} Li}]{9407870}
Deng H, Yin S, Deng X, Li S (2020) Value-based algorithms optimization with
  discounted multiple-step learning method in deep reinforcement learning.
  \emph{2020 IEEE 22nd International Conference on High Performance Computing
  and Communications; IEEE 18th International Conference on Smart City; IEEE
  6th International Conference on Data Science and Systems
  (HPCC/SmartCity/DSS)}, 979--984,
  \urlprefix\url{http://dx.doi.org/10.1109/HPCC-SmartCity-DSS50907.2020.00131}.

\bibitem[{Du et~al.(2018)Du, Zhai, Poczos, \protect\BIBand{}
  Singh}]{du2018gradient}
Du SS, Zhai X, Poczos B, Singh A (2018) Gradient descent provably optimizes
  over-parameterized neural networks. \emph{International Conference on
  Learning Representations}.

\bibitem[{Efroni et~al.(2018)Efroni, Dalal, Scherrer, \protect\BIBand{}
  Mannor}]{efroni2018multiplestep}
Efroni Y, Dalal G, Scherrer B, Mannor S (2018) Multiple-step greedy policies in
  online and approximate reinforcement learning.

\bibitem[{Efroni et~al.(2019)Efroni, Dalal, Scherrer, \protect\BIBand{}
  Mannor}]{efroni2019combine}
Efroni Y, Dalal G, Scherrer B, Mannor S (2019) How to combine tree-search
  methods in reinforcement learning.

\bibitem[{Efroni et~al.(2020)Efroni, Ghavamzadeh, \protect\BIBand{}
  Mannor}]{efroni2020online}
Efroni Y, Ghavamzadeh M, Mannor S (2020) Online planning with lookahead
  policies. \emph{Advances in Neural Information Processing Systems} 33.

\bibitem[{Jacot et~al.(2018)Jacot, Gabriel, \protect\BIBand{}
  Hongler}]{jacot2018neural}
Jacot A, Gabriel F, Hongler C (2018) Neural tangent kernel: Convergence and
  generalization in neural networks. \emph{arXiv preprint arXiv:1806.07572} .

\bibitem[{Ji \protect\BIBand{} Telgarsky(2019)}]{ji2019polylogarithmic}
Ji Z, Telgarsky M (2019) Polylogarithmic width suffices for gradient descent to
  achieve arbitrarily small test error with shallow relu networks.
  \emph{International Conference on Learning Representations}.

\bibitem[{Kocsis \protect\BIBand{} Szepesvári(2006)}]{kocisszepesvari}
Kocsis L, Szepesvári C (2006) Bandit based monte-carlo planning. \emph{Machine
  Learning: ECML}, volume 2006, 282--293, ISBN 978-3-540-45375-8,
  \urlprefix\url{http://dx.doi.org/10.1007/11871842_29}.

\bibitem[{Lagoudakis \protect\BIBand{} Parr(2001)}]{lagoudakis2001model}
Lagoudakis MG, Parr R (2001) Model-free least-squares policy iteration.
  \emph{Advances in neural information processing systems} 14.

\bibitem[{Lagoudakis \protect\BIBand{} Parr(2003)}]{parr}
Lagoudakis MG, Parr R (2003) Least-squares policy iteration. \emph{The Journal
  of Machine Learning Research} 4:1107--1149.

\bibitem[{Lanctot et~al.(2014)Lanctot, Winands, Pepels, \protect\BIBand{}
  Sturtevant}]{lanctot2014monte}
Lanctot M, Winands MHM, Pepels T, Sturtevant NR (2014) Monte carlo tree search
  with heuristic evaluations using implicit minimax backups.

\bibitem[{Mnih et~al.(2016)Mnih, Badia, Mirza, Graves, Lillicrap, Harley,
  Silver, \protect\BIBand{} Kavukcuoglu}]{DBLP:journals/corr/MnihBMGLHSK16}
Mnih V, Badia AP, Mirza M, Graves A, Lillicrap TP, Harley T, Silver D,
  Kavukcuoglu K (2016) Asynchronous methods for deep reinforcement learning.
  \emph{CoRR} abs/1602.01783, \urlprefix\url{http://arxiv.org/abs/1602.01783}.

\bibitem[{Moerland et~al.(2020)Moerland, Broekens, \protect\BIBand{}
  Jonker}]{moerland2020framework}
Moerland TM, Broekens J, Jonker CM (2020) A framework for reinforcement
  learning and planning.

\bibitem[{Munos(2014)}]{munosbook}
Munos R (2014) From bandits to monte-carlo tree search: The optimistic
  principle applied to optimization and planning. \emph{Foundations and Trends
  in Machine Learning} 7, \urlprefix\url{http://dx.doi.org/10.1561/2200000038}.

\bibitem[{Puterman \protect\BIBand{} Shin(1978)}]{Puterman1978ModifiedPI}
Puterman M, Shin MC (1978) Modified policy iteration algorithms for discounted
  markov decision problems. \emph{Management Science} 24:1127--1137.

\bibitem[{Shah et~al.(2020{\natexlab{a}})Shah, Somani, Xie, \protect\BIBand{}
  Xu}]{shahxie}
Shah D, Somani V, Xie Q, Xu Z (2020{\natexlab{a}}) On reinforcement learning
  for turn-based zero-sum markov games. \emph{CoRR} abs/2002.10620,
  \urlprefix\url{https://arxiv.org/abs/2002.10620}.

\bibitem[{Shah et~al.(2020{\natexlab{b}})Shah, Xie, \protect\BIBand{}
  Xu}]{shah2020nonasymptotic}
Shah D, Xie Q, Xu Z (2020{\natexlab{b}}) Non-asymptotic analysis of monte carlo
  tree search.

\bibitem[{Silver et~al.(2017{\natexlab{a}})Silver, Hubert, Schrittwieser,
  Antonoglou, Lai, Guez, Lanctot, Sifre, Kumaran, Graepel, Lillicrap, Simonyan,
  \protect\BIBand{} Hassabis}]{silver2017shoji}
Silver D, Hubert T, Schrittwieser J, Antonoglou I, Lai M, Guez A, Lanctot M,
  Sifre L, Kumaran D, Graepel T, Lillicrap TP, Simonyan K, Hassabis D
  (2017{\natexlab{a}}) Mastering chess and shogi by self-play with a general
  reinforcement learning algorithm. \emph{CoRR} abs/1712.01815,
  \urlprefix\url{http://arxiv.org/abs/1712.01815}.

\bibitem[{Silver et~al.(2017{\natexlab{b}})Silver, Schrittwieser, Simonyan,
  Antonoglou, Huang, Guez, Hubert, Baker, Lai, Bolton
  et~al.}]{silver2017mastering}
Silver D, Schrittwieser J, Simonyan K, Antonoglou I, Huang A, Guez A, Hubert T,
  Baker L, Lai M, Bolton A, et~al. (2017{\natexlab{b}}) Mastering the game of
  go without human knowledge. \emph{Nature} 550(7676):354--359.

\bibitem[{Springenberg et~al.(2020)Springenberg, Heess, Mankowitz, Merel,
  Byravan, Abdolmaleki, Kay, Degrave, Schrittwieser, Tassa
  et~al.}]{springenberg2020local}
Springenberg JT, Heess N, Mankowitz D, Merel J, Byravan A, Abdolmaleki A, Kay
  J, Degrave J, Schrittwieser J, Tassa Y, et~al. (2020) Local search for policy
  iteration in continuous control. \emph{arXiv preprint arXiv:2010.05545} .

\bibitem[{Tomar et~al.(2020)Tomar, Efroni, \protect\BIBand{}
  Ghavamzadeh}]{tomar2020multistep}
Tomar M, Efroni Y, Ghavamzadeh M (2020) Multi-step greedy reinforcement
  learning algorithms.

\bibitem[{Tsitsiklis \protect\BIBand{} Roy(1994)}]{TsitsiklisRoy}
Tsitsiklis JN, Roy BV (1994) Feature-based methods for large scale dynamic
  programming. \emph{Machine Learning}, 59--94.

\bibitem[{Tsitsiklis \protect\BIBand{} van
  Roy(1994)}]{Tsitsiklis94feature-basedmethods}
Tsitsiklis JN, van Roy B (1994) Feature-based methods for large scale dynamic
  programming. \emph{Machine Learning}, 59--94.

\bibitem[{Veness et~al.(2009)Veness, Silver, Blair, \protect\BIBand{}
  Uther}]{veness}
Veness J, Silver D, Blair A, Uther W (2009) Bootstrapping from game tree
  search. Bengio Y, Schuurmans D, Lafferty J, Williams C, Culotta A, eds.,
  \emph{Advances in Neural Information Processing Systems}, volume~22 (Curran
  Associates, Inc.),
  \urlprefix\url{https://proceedings.neurips.cc/paper/2009/file/389bc7bb1e1c2a5e7e147703232a88f6-Paper.pdf}.

\end{thebibliography}
\end{document}